\documentclass{article} 
\usepackage{iclr2026_conference,times}
\usepackage{amsthm}

\newtheorem{proposition}{Proposition}

\usepackage{amsmath,amsfonts,bm}

\def\eqref#1{equation~\ref{#1}}

\def\1{\bm{1}}

\def\rvx{{\mathbf{x}}}

\def\rvz{{\mathbf{z}}}

\DeclareMathAlphabet{\mathsfit}{\encodingdefault}{\sfdefault}{m}{sl}
\SetMathAlphabet{\mathsfit}{bold}{\encodingdefault}{\sfdefault}{bx}{n}

\newcommand{\E}{\mathbb{E}}

\newcommand{\KL}{D_{\mathrm{KL}}}

\DeclareMathOperator*{\argmin}{arg\,min}

\usepackage{hyperref}
\usepackage{url}
\usepackage{tikz}
\usetikzlibrary{positioning, fit, calc, arrows.meta, shapes}
\usetikzlibrary{automata,arrows}
\usetikzlibrary{shapes.multipart}
\usetikzlibrary{fit, backgrounds, decorations.pathreplacing}
\usepackage[capitalize]{cleveref}
\usepackage{booktabs}
\usepackage{multirow}
\makeatletter
\renewcommand{\@fnsymbol}[1]{\ifcase#1\or\dagger\or\ddagger\else\@ctrerr\fi}
\makeatother

\title{Inference-Time Rethinking with\\ Latent Thought Vectors for Math Reasoning}

\author{Deqian Kong$^{1,2,}$\thanks{Equal Contribution, $^\ddagger$Equal Advising. Work done while D.~K. was an intern at Lambda, Inc.}~, Minglu Zhao$^{1,\dagger}$, Aoyang Qin$^{3,\dagger}$, Bo Pang$^{4}$, Chenxin Tao$^{3}$, \\ \textbf{David Hartmann}$^{2}$,  \textbf{Edouardo Honig}$^{1}$\textbf{,} \textbf{Dehong Xu}$^{1}$\textbf{,} \textbf{Amit Kumar}$^{2}$\textbf{,} \textbf{Matt Sarte}$^{2}$\textbf{,} \\\textbf{Chuan Li}$^{2}$\textbf{,}  \textbf{Jianwen Xie}$^{2,\ddagger}$\textbf{,} \textbf{Ying Nian Wu$^{1,\ddagger}$}\\
$^{1}$ Department of Statistics and Data Science, UCLA \\
$^{2}$ Lambda Inc, $^{3}$ Tsinghua University, $^{4}$ Salesforce Research\\
}

\usepackage{enumitem}
\definecolor{customlinkcolor}{HTML}{2774AE} 
\definecolor{customcitecolor}{HTML}{2774AE} 

\hypersetup{
    colorlinks=true,
    linkcolor=customlinkcolor, 
    anchorcolor=black, 
    citecolor=customcitecolor  
}

\iclrfinalcopy 
\begin{document}

\maketitle

\begin{abstract}
Standard chain-of-thought reasoning generates a solution in a single forward pass, committing irrevocably to each token and lacking a mechanism to recover from early errors. We introduce \emph{Inference-Time Rethinking}, a generative framework that enables iterative self-correction by decoupling declarative \textit{latent thought vectors} from procedural generation. We factorize reasoning into a continuous latent thought vector (what to reason about) and a decoder that verbalizes the trace conditioned on this vector (how to reason). Beyond serving as a declarative buffer, latent thought vectors compress the reasoning structure into a continuous representation that abstracts away surface-level token variability, making gradient-based optimization over reasoning strategies well-posed. Our prior model maps unstructured noise to a learned manifold of valid reasoning patterns, and at test time we employ a Gibbs-style procedure that alternates between generating a candidate trace and optimizing the latent vector to better explain that trace, effectively navigating the latent manifold to refine the reasoning strategy. Training a 0.2B-parameter model from scratch on GSM8K, our method with 30 rethinking iterations surpasses baselines with 10--15$\times$ more parameters, including a 3B counterpart. This result demonstrates that effective mathematical reasoning can emerge from sophisticated inference-time computation rather than solely from massive parameter counts. 
\end{abstract}

\section{Introduction}

Human cognition fundamentally distinguishes between \textbf{declarative knowledge} and \textbf{procedural execution}. The declarative system stores facts, episodes, and explicit plans, while the procedural system encodes the skills and routines required to act upon them~\citep{ullman2004contributions}. This functional separation is mirrored in the brain's neural architecture: for example, Wernicke's area processes semantic representations (the meaning of concepts), whereas Broca's area orchestrates syntactic production (the procedural act of speech)~\citep{friederici2011brain}.  This separation allows the brain to hold a ``thought'' in a declarative buffer, inspect it, and revise it before committing to a procedural output. Ideally, a reasoning system should preserve this distinction: maintaining a separation between the \emph{content} of a thought and the \emph{process} of expressing it.

Modern language models collapse this distinction. In a standard Transformer, the semantic content (declarative) and the syntactic realization (procedural) are entangled within the same parameter matrix, retrieved through a single, irrevocable forward pass. There is no independent locus for the reasoning plan—no latent workspace that can be critiqued or refined. Consequently, errors in early generation steps propagate unchecked, as the model lacks the mechanism to backtrack and rethink.

We propose to restore this cognitive separation via \textbf{latent thought vectors}. We factorize reasoning into two components: a {prior encoder} that maps random noise tokens into structured latent thought vectors $\rvz$ (declarative: what to reason about), and a decoder that generates token sequences conditioned on $\rvz$ (procedural: how to verbalize reasoning). Beyond serving as a declarative buffer, latent
thought vectors play a \emph{formalization} role: they compress the
reasoning structure into a continuous representation that abstracts away
surface-level token variability, making gradient-based optimization over
reasoning strategies well-posed.

Crucially, our Transformer encoder maps unstructured noise to a learned
manifold whose geometry reflects valid reasoning patterns, allowing us to
formulate \textbf{reasoning as gradient-based optimization in the space
of valid thoughts}. We introduce \textbf{Inference-Time Rethinking}, a
Gibbs-style procedure that alternates between generation and reflection:
(1)~\textit{Generate:} The decoder produces a candidate reasoning trace
conditioned on the current latent thought $\rvz$.
(2)~\textit{Reflect:} We optimize the input noise via gradient descent
to maximize the likelihood of the generated trace, moving the thought
vector toward a region of the manifold that better explains the desired
logic. This loop transforms inference from a single-pass commitment into
a dynamic optimization trajectory, allowing the model to recover from
errors by navigating the latent space.

This framework also embodies the Complementary Learning Systems theory~\citep{mcclelland1995complementary}. Our dual-rate training scheme mirrors the interplay between fast and slow learning: we perform \textbf{fast} optimization of the latent $\rvz$ for each instance (adapting to the specific problem) and \textbf{slow} updates to the model parameters $\theta$ (accumulating general schema). At test time, the fast pathway remains active, enabling the model to invest compute in per-instance refinement.

By scaling this inference-time computation, we achieve significant performance gains without increasing model size. On GSM8K, our 0.2B-parameter model with 30 rethinking iterations achieves 31.5\% accuracy, outperforming 10--15$\times$ larger models including CoT-SFT on a 3B counterpart (22.7\%) and MARCoS-2B (24.1\%). The gains extend to out-of-domain benchmarks, reaching 51.5\% on SVAMP and 68.0\% on MultiArith. These results demonstrate that effective mathematical reasoning can emerge from the ability to iteratively optimize declarative thoughts, rather than solely from larger parameter counts.

Our contributions are:
\begin{itemize}[itemsep=2pt, topsep=4pt, leftmargin=18pt]
    \item A generative framework that decouples declarative content ($\rvz$) from procedural generation $\theta$. This separation yields efficient learning: a 0.2B model outperforms much larger monolithic architectures even in a single pass, with strong out-of-domain robustness.
    \item \emph{Inference-Time Rethinking}, a Gibbs-style procedure that iteratively refines the latent thought by alternating between trace generation and latent optimization, enabling the model to self-correct at test time.
    \item Empirical validation showing that a 0.2B model with rethinking outperforms 10--15$\times$ larger baselines, establishing inference-time computation as a highly effective scaling axis orthogonal to parameter count.
\end{itemize}
\section{Method}

\subsection{Generative Model with Latent Thought Vectors}
\label{sec:model}

Let $\rvx = (x^{(1)}, \dots, x^{(N)})$ denote a sequence of ground tokens, which may include a question $\rvx_q$ and reasoning trace $\rvx_r$. We introduce latent thought vectors $\rvz \in \mathbb{R}^d$ that guide generation via the joint model
\begin{equation}
    p_\theta(\rvx, \rvz) = p_\alpha(\rvz)\, p_\beta(\rvx | \rvz),
\end{equation}
where $p_\alpha(\rvz)$ is a learnable prior with parameters $\alpha$, and $p_\beta(\rvx | \rvz)$ is an autoregressive decoder with parameters $\beta$. Here $\theta = (\alpha, \beta)$ denotes all model parameters. Unlike standard language models that condition only on preceding tokens, $\rvz$ cross-attends to every decoder layer, providing a global conditioning signal throughout generation (\cref{fig:model}).

\begin{figure}[h]
\centering
\resizebox{0.5\linewidth}{!}{
\begin{tikzpicture}[->, >=stealth', auto, thick, node distance=1.3cm]

    \begin{scope}[shift={(2cm,-0.75cm)}, 
      block/.style={rectangle, draw, fill=black!10, minimum width=3cm, minimum height=0.6cm, align=center, rounded corners},
      arrow/.style={thin, -stealth'},
      bigbox/.style={draw, thick, rounded corners, inner sep=5pt, dashed}]
      
        \node[block] (sa1) {\small Self-attention};

        \begin{scope}[on background layer]
            \node[bigbox, fit=(sa1) ] (mainbox_left) {};
        \end{scope}

        \node[above right=0.1cm and -1cm of sa1.north east, anchor=south west] (N_left) {$\times N_\mathrm{Enc}$};

        \node[below=0.2cm of mainbox_left] (z0) {$\rvz_0$};

        \node[above=0.26cm of mainbox_left] (z_in) {$\rvz=(\rvz_1,...,\rvz_K)$};

        \draw[arrow] (z0) -- (sa1);
        \draw[arrow] (sa1) -- (z_in);

        \node[above=1cm of mainbox_left] {Transformer Encoder};

    \end{scope}

    \draw[-,dashed] (4.25, -1.8) -- (4.25, 2.9);

    \begin{scope}[shift={(6.5cm,1cm)}, 
      block/.style={rectangle, draw, fill=black!10, minimum width=3cm, minimum height=0.6cm, align=center, rounded corners},
      arrow/.style={thin, -stealth'},
      bigbox/.style={draw, thick, rounded corners, inner sep=5pt, dashed}]
    
        \node[block] (cross) {\small Cross-attention};
        \node[block, below=0.82cm of cross] (causal) {\small Self-attention};

        \begin{scope}[on background layer]
            \node[bigbox, fit=(cross) (causal)] (mainbox) {};
        \end{scope}

        \node[above right=0.1cm and -1cm of cross.north east, anchor=south west] (N) {$\times N_\mathrm{Dec}$};

        \node[below=0.32 cm of mainbox] (xi) {\small$x^{(0)},\dots,x^{(T-1)}$};
        \node[above=0.2 cm of mainbox] (xo) {\small$x^{(1)},\dots,x^{(T)}$};

        \draw[arrow] (z_in) -| ($(cross.south west)!0.25!(cross.south east)$);
        \draw[arrow] (z_in) -| ($(cross.south west)!0.5!(cross.south east)$);
        \draw[arrow] (xi) -- (causal);
        \draw[arrow] (cross) -- (xo);

        \draw[arrow] ($(causal.north west)!0.75!(causal.north east)$) -- ($(cross.south west)!0.75!(cross.south east)$);

        \node[above=1cm of mainbox] {Transformer Decoder};

    \end{scope}

\end{tikzpicture}
}
\caption{\textbf{Model architecture.} The encoder transforms noise tokens $\rvz_0 \sim \mathcal{N}(\mathbf{0}, \mathbf{I})$ into latent thought vectors $\rvz = U_\alpha(\rvz_0)$. The decoder generates ground tokens $\rvx$ conditioned on $\rvz$ via cross-attention at each layer. Posterior inference operates on $\rvz_0$ in the base space.}
\label{fig:model}
\vspace{-10pt}
\end{figure}
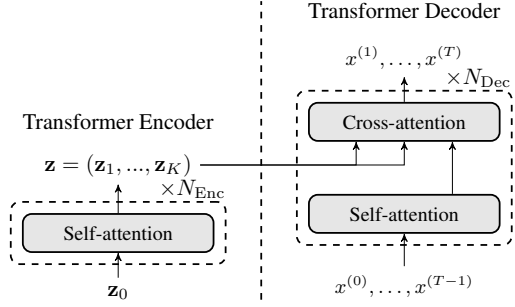

{\bf Prior Model.}
We parameterize the prior via a transport map:
\begin{equation}
\label{eq:prior}
    \rvz_0 \sim \mathcal{N}(\mathbf{0}, \mathbf{I}), \qquad \rvz = U_\alpha(\rvz_0),
\end{equation}
where $U_\alpha$ is a Transformer encoder with bidirectional attention~\citep{vaswani2017attention}. The encoder transforms noise tokens $\rvz_0$ into structured thought vectors $\rvz$; posterior inference over $\rvz$ reduces to inference over $\rvz_0$ in
the base space. Intuitively, $\rvz$ represents thought patterns instantiated as multiple tokens, while $\rvz_0$ indexes into the space of reasoning strategies encoded in the encoder parameters $\alpha$.

{\bf Thought-Guided Generator.}
The generator $p_\beta(\rvx | \rvz)$ is a Transformer decoder that incorporates thought vectors at each generation step via cross-attention:
\begin{equation}
    p_\beta(\rvx | \rvz) = \prod_{n=1}^{N} p_\beta(x^{(n)} | \rvz, x^{(n-w:n-1)}),
\end{equation}
where $x^{(n-w:n-1)}$ denotes the $w$ preceding tokens. We use a short context window ($w = 64$) so that the decoder cannot attend directly beyond 64 positions; long-range dependencies must route through $\rvz$. This architectural bottleneck encourages $\rvz$ to encode global, sequence-level structure rather than redundant local context.

Our architecture instantiates the encoder-decoder Transformer~\citep{vaswani2017attention}, repurposed for latent variable modeling. The thought vectors $\rvz$ serve as instance-specific local parameters inferred per sequence, while $\theta = (\alpha, \beta)$ comprises global parameters learned from all data.

\subsection{Learning and Inference}
\label{sec:learning}

We maximize a variational lower bound on the marginal log-likelihood $\log p_\theta(\rvx) = \log \int p_\beta(\rvx | \rvz=U_\alpha(\rvz_0)) p(\rvz_0)\, d\rvz_0$. For each sequence $\rvx$, we introduce a Gaussian variational posterior over the noise tokens $q(\rvz_0) = \mathcal{N}(\boldsymbol{\mu}, \mathrm{diag}(\boldsymbol{\sigma}^2))$,
with local parameters $(\boldsymbol{\mu}, \boldsymbol{\sigma}^2)$ specific to that sequence. The evidence lower bound (ELBO) is:
\begin{equation}
\label{eq:elbo}
    \mathcal{L}(\theta, \boldsymbol{\mu}, \boldsymbol{\sigma}^2) = \mathbb{E}_{q(\rvz_0)}\bigl[\log p_\beta(\rvx | \rvz=U_\alpha(\rvz_0))\bigr] - \KL\bigl(q(\rvz_0) \,\|\, p(\rvz_0)\bigr),
\end{equation}
optimized via the reparameterization trick~\citep{kingma2013auto}, and $p(\rvz_0)=\mathcal{N}(\mathbf{0}, \mathbf{I})$.

Crucially, $(\boldsymbol{\mu}, \boldsymbol{\sigma}^2)$ is \emph{not} amortized: we do not train an inference network. Instead, each sequence maintains its own $(\boldsymbol{\mu}, \boldsymbol{\sigma}^2)$, initialized from the prior and optimized directly. This classical variational Bayes approach~\citep{jordan1999introduction, blei2017variational} avoids the posterior collapse often observed in VAEs~\citep{bowman2016generating} for language modeling and provides per-instance precision.

{\bf Dual-Rate Optimization.}
Training alternates between two timescales:
\begin{itemize}[itemsep=2pt, topsep=2pt, leftmargin=18pt]
    \item \textbf{Fast local updates}: Optimize $ (\boldsymbol{\mu}, \boldsymbol{\sigma}^2)$ for $T_\mathrm{fast}$ steps using Adam~\citep{kingma2014adam} with a high learning rate (e.g., 0.3).
    \item \textbf{Slow global updates}: Update $\theta = (\alpha, \beta)$ with a standard learning rate (e.g., $4 \times 10^{-4}$).
\end{itemize}
This dual-rate scheme enables rapid per-instance adaptation of the latent representation while gradually accumulating shared knowledge in the encoder and decoder parameters.

At test time, inference proceeds identically: given a new sequence, we initialize $(\boldsymbol{\mu}, \boldsymbol{\sigma}^2)$ from the prior and run $T_\mathrm{fast}$ optimization steps. We use $T_\mathrm{fast} = 16$ throughout, which we find sufficient for convergence given the model's architectural constraints (short context window, moderate latent dimension). This finite-step inference-time computation is the foundation for the rethinking procedure described next.

{\bf Comparison with Latent Thought Models (LTMs).} This framework extends~\citet{kong2025latent}, who used a fixed prior $p(\rvz) = \mathcal{N}(\mathbf{0}, \mathbf{I})$ and learned only decoder parameters. Our transport-based prior $\rvz = U_\alpha(\rvz_0)$ introduces learnable encoder parameters $\alpha$, yielding a latent space whose geometry is shaped by training. This learned structure proves critical for inference-time rethinking (\cref{sec:itr}): gradient-based refinement is more effective when operating in a manifold that reflects the distribution of valid reasoning patterns.

\subsection{Inference-Time Rethinking}
\label{sec:itr}

A standard autoregressive LLM generates a reasoning trace in a single forward pass: $\rvx_r \sim p_\beta(\rvx_r | \rvx_q)$. Once tokens are committed, errors propagate and the model cannot revisit earlier reasoning steps. We propose \emph{inference-time rethinking}, an iterative procedure that enables revision by alternating between generation and latent reflection.

Given a question $\rvx_q$, the goal is to sample from the marginal
\begin{equation}
    p_\theta(\rvx_r | \rvx_q) = \int p_\beta(\rvx_r |U_\alpha(\rvz_0), \rvx_q)\, p(\rvz_0|\rvx_q)\, d\rvz_0\approx \E_{q(\rvz_0)}[p_\beta(\rvx_r |U_\alpha(\rvz_0), \rvx_q].
\end{equation}
Since direct marginalization is intractable, we employ a Gibbs-style procedure that alternates between trace generation and latent refinement:
\begin{align} 
&\text{Step (1) }\textit{Generate:} \quad  \rvx_r^{(t)} \sim p_\beta(\rvx_r | \rvz^{(t-1)}= U_\alpha(\rvz_0^{(t-1)}), \rvx_q), \label{eq:itr-step1} \\
&\text{Step (2) }\textit{Reflect:} \qquad  \rvz_0^{(t)} = \argmin_{\rvz_0} \KL\bigl(q(\rvz_0 | \rvx_r^{(t)}, \rvx_q) \,\|\, p_\theta(\rvz_0 | \rvx_r^{(t)}, \rvx_q)\bigr). \label{eq:itr-step2} 
\end{align}
Step (1) generates a candidate reasoning trace given the current latent thought. Step (2) updates the latent to better explain the generated trace, effectively critiquing the thought vectors in latent space. Unlike autoregressive generation, which commits irrevocably to each token, this loop allows the model to \textbf{revise past reasoning}: errors in $\rvx_r^{(t)}$ inform a refined $\rvz_0^{(t)}$, which then guides a corrected trace $\rvx_r^{(t+1)}$. In practice, Step (2) is optimized by maximizing the ELBO in \cref{eq:elbo} with $(\rvx_q, \rvx_r^{(t)})$ as the observed sequence.

We initialize by running variational inference on the question alone to obtain $\rvz_0^{(0)}$, then iterate Steps (1)--(2) for $T_\mathrm{rethink}$ rounds. Since the rethinking trajectory explores multiple regions of the latent manifold, we retain the trace with the highest likelihood, and extract the final answer from it.

\section{Experiments}
\label{sec:exp}
\subsection{Setup}
{\bf Training Data.}
Following~\citet{tan2025think}, we train all models from scratch on GSM8K-Aug~\citep{deng2023implicit}, an augmented version of GSM8K~\citep{cobbe2021training} containing 385K training examples. Each example pairs a question with an equation-based chain-of-thought solution. Crucially, no model uses pretrained weights, ensuring a fair comparison of learning efficiency rather than transfer from pretraining.

{\bf Evaluation Benchmarks.}
Following~\citet{liu2025marcos}, we evaluate on three benchmarks: (1)~\textbf{GSM8K}~\citep{cobbe2021training}, the original test set of 1,319 grade-school math problems (in-domain); (2)~\textbf{SVAMP}~\citep{patel2021nlp}, 4,138 arithmetic problems constructed via subtle variations in wording and semantic perturbations, testing robustness to surface-form changes (out-of-domain); (3)~\textbf{MultiArith}~\citep{roy2015solving}, 600 problems from MAWPS~\citep{koncel2016mawps} requiring multi-step arithmetic reasoning (out-of-domain). We report accuracy as the evaluation metric.

{\bf Baselines.}
We compare against CoT supervised fine-tuning (CoT-SFT)~\citep{wei2022chain} and recent latent reasoning methods: iCoT-SI~\citep{deng2024explicit}, Coconut~\citep{hao2024training}, CoLaR~\citep{tan2025think}, CODI~\citep{shen2025codi}, and MARCoS~\citep{liu2025marcos}. Baselines use Qwen2.5-0.5B or Qwen2.5-3B~\citep{team2024qwen2} as backbones; MARCoS uses a 2B model. All baseline numbers are from~\citet{liu2025marcos}.

{\bf Our Model.}
We use a 0.2B-parameter Transformer following the
Llama-2~\citep{touvron2023llama} architecture with a context window of $w = 64$, a hidden size of 1024, an 8-layer decoder, and a 2-layer encoder with 64 latent tokens. We run $T_\mathrm{fast} = 16$ inference steps for latent optimization. Note that our model is smaller than every baseline. We report results for single-pass generation (Rethink-1) and iterative rethinking with $T_\mathrm{rethink} = 30$ rounds (Rethink-30).

\subsection{Main Results}
\begin{table}[h]
\centering
\caption{Accuracy (\%) on mathematical reasoning benchmarks. All models
trained from scratch on equation-format GSM8K-Aug training set. Best in
\textbf{bold}, second-best \underline{underlined}.}
\label{tab:main}

\resizebox{0.7\linewidth}{!}{
\begin{tabular}{llccc}
\toprule
\textbf{Backbone} & \textbf{Method} & \textbf{GSM8K} & \textbf{SVAMP} & \textbf{MultiArith} \\
\midrule
\multirow{5}{*}{Qwen2.5-0.5B}
& CoT-SFT       & 20.31 & 27.92 & 41.48 \\
& iCoT-SI       & 14.81 & 18.46 & 21.33 \\
& Coconut       & 15.27 & 17.76 & 15.50 \\
& CoLaR         & 15.45 & 23.09 & 38.60 \\
& CODI          &  2.55 &  1.66 &  1.60 \\
& MARCoS (2B)   & 24.11 & 27.77 & 42.33 \\
\midrule
\multirow{6}{*}{Qwen2.5-3B}
& CoT-SFT       & 22.70 & 27.25 & 50.20 \\
& iCoT-SI       & 14.08 & 16.26 & 19.67 \\
& Coconut       &  9.95 & 13.53 &  8.78 \\
& CoLaR         & 22.60 & 25.83 & 41.83 \\
& CODI          &  1.24 &  1.12 &  1.20 \\
\midrule
\multirow{2}{*}{Ours (0.2B)}
& Rethink-1     & \underline{25.93} & \underline{47.37} & \underline{63.00} \\
& Rethink-30    & \textbf{31.54} & \textbf{51.50} & \textbf{68.00} \\
\bottomrule
\end{tabular}
}
\end{table}

\cref{tab:main} presents the results. Even without rethinking,
Rethink-1 already surpasses nearly all baselines, including models with
up to $15\times$ more parameters. This demonstrates the learning
efficiency gained by separating the latent thought vectors from the
decoder: offloading reasoning structure to $\rvz$ frees the decoder from
memorizing all patterns in its parameters, allowing a 0.2B model to
outcompete much larger monolithic architectures. With 30 rethinking
iterations, Rethink-30 further improves to 31.54\% on GSM8K, 51.50\% on
SVAMP, and 68.00\% on MultiArith, establishing new bests across all
three benchmarks with the smallest model in the comparison.

{\bf Value of Rethinking.}
The gap between Rethink-1 and Rethink-30 quantifies the benefit of
iterative refinement: 5.6 points on GSM8K, 4.1 on SVAMP, and 5.0 on
MultiArith, all from the same model with additional test-time
computation. This confirms that inference-time computation is an
effective scaling axis: a smaller model that ``thinks longer'' surpasses a larger model that generates in one pass.

{\bf Out-of-Domain Generalization.}
The gains on SVAMP and MultiArith are particularly striking. Because the
latent thought vector abstracts away surface-level token patterns, the
model's reasoning strategy transfers across different problem wordings
(SVAMP) and multi-step structures (MultiArith). Rethinking amplifies
this effect: iterative optimization in latent space refines the
underlying reasoning plan without overfitting to surface form, improving
robustness to distribution shift.

\section{Discussion}

\textbf{Likelihood as a Proxy for Correctness.} Our inference-time rethinking procedure refines the latent thought vector by maximizing the likelihood of the generated reasoning trace. This approach is effective because our model is trained on near-expert demonstrations: the learned distribution closely approximates correct reasoning, so the likelihood landscape naturally favors valid logical paths. A latent vector that produces a high-likelihood trace is statistically likely to encode a sound reasoning strategy. By iteratively optimizing the latent vector to better explain its own high-confidence generations, the model gravitates toward reasoning modes that are implicitly defined by the training data.

\textbf{Generalizing Beyond Clean Supervision.} This reliance on likelihood as a proxy for correctness has a clear limitation: it assumes high-quality training data. When the model is trained on uncurated or noisy supervision, the correlation between likelihood and correctness weakens, as high-likelihood traces may simply reflect common errors or misconceptions rather than valid reasoning. In such settings, relying solely on the generative prior and the decoder is insufficient for robust self-correction, and additional signals are needed to guide the rethinking process.

\textbf{Towards Latent Planning and Verification.}
A natural next step is to incorporate explicit value signals into the
rethinking loop. One direction is to introduce a \textit{latent
verifier} that predicts the correctness of a reasoning strategy
directly from the thought vector, before any tokens are decoded. This
would transform rethinking into a form of latent
planning~\citep{kong2024molecule,kong2024latent,qin2025generative},
where the model evaluates and selects among candidate strategies in
the continuous latent space. Another promising direction is to
leverage discrete external rewards, such as feedback from a code
compiler or a symbolic checker, to perform \textit{test-time policy
gradient optimization}~\citep{li2025latentseek} over the latent thought
vector. Regularizing toward the learned
prior naturally serves as a trust region~\citep{schulman2017proximal,noh2025latent}, keeping the optimization
within the manifold where the decoder remains well-calibrated and
bridging intuitive generation with deliberative, reward-guided search.

\section{Conclusion}
We introduced Inference-Time Rethinking, a framework that decouples
reasoning into latent thought vectors and procedural generation, and
iteratively refines the thought vector at test time via a Gibbs-style
loop. Offloading reasoning structure to latent variables enables
efficient learning, a 0.2B model already surpasses much larger
baselines in a single pass, while the continuous, training-shaped
latent space supports gradient-based self-correction through additional
compute. On GSM8K, SVAMP, and MultiArith, our model establishes new
bests across all benchmarks while remaining the smallest in the
comparison, demonstrating that inference-time computation over a
well-structured latent space is a powerful scaling axis complementary
to parameter count.

\section*{Acknowledgments}
Y.~W. is partially supported by NSF DMS-2415226, DARPA W912CG25CA007 and research gift funds from Amazon and Qualcomm. We thank Luca Zancato and Ben Bowman of Amazon AWS Research for their valuable help and discussions. We gratefully acknowledge the
support of Lambda, Inc. for providing the computational resources used in this project.

\appendix

\section{Theoretical Understanding of Inference-Time Rethinking}
\label{sec:theory}

The inference-time rethinking procedure described in \cref{sec:itr} alternates between generating a reasoning trace and refining the latent thought vector. In this section, we show that this alternating procedure admits a principled variational interpretation: it performs coordinate-ascent optimization on a variational lower bound of $\log p_\theta(\rvx_q)$, equivalently minimizing the KL divergence between a factored variational distribution and the intractable joint posterior $p_\theta(\rvx_r, \rvz_0 \mid \rvx_q)$.

\subsection{Variational Formulation}

Given a question $\rvx_q$, our generative model defines a joint distribution over reasoning traces $\rvx_r$ and latent noise tokens $\rvz_0$:
\begin{equation}
\label{eq:joint}
    p_\theta(\rvx_r, \rvz_0 \mid \rvx_q) = \frac{p_\beta(\rvx_r \mid U_\alpha(\rvz_0), \rvx_q)\, p(\rvz_0)}{p_\theta(\rvx_q)},
\end{equation}
where $p(\rvz_0) = \mathcal{N}(\mathbf{0}, \mathbf{I})$ is the base-space prior. Ideally, we would sample from this joint posterior to obtain both a good reasoning strategy ($\rvz_0$) and a corresponding trace ($\rvx_r$). However, the marginal $p_\theta(\rvx_q) = \int \sum_{\rvx_r} p_\beta(\rvx_r \mid U_\alpha(\rvz_0), \rvx_q)\, p(\rvz_0)\, d\rvz_0$ is intractable.

We introduce a \textbf{mean-field variational approximation} that factorizes over the two latent variables:
\begin{equation}
\label{eq:mean-field}
    q(\rvx_r, \rvz_0) = q_1(\rvx_r)\, q_2(\rvz_0),
\end{equation}
and seek to minimize the KL divergence to the true posterior:
\begin{equation}
\label{eq:kl-objective}
    \min_{q_1, q_2} \; \KL\bigl(q_1(\rvx_r)\, q_2(\rvz_0) \;\|\; p_\theta(\rvx_r, \rvz_0 \mid \rvx_q)\bigr).
\end{equation}

By the standard identity $\log p_\theta(\rvx_q) = \mathcal{L}(q_1, q_2) + \KL(q_1 q_2 \| p_\theta(\rvx_r, \rvz_0 \mid \rvx_q))$, minimizing \cref{eq:kl-objective} is equivalent to maximizing the evidence lower bound (ELBO):
\begin{equation}
\label{eq:elbo-joint}
    \mathcal{L}(q_1, q_2) = \mathbb{E}_{q_1(\rvx_r) q_2(\rvz_0)}\bigl[\log p_\beta(\rvx_r \mid U_\alpha(\rvz_0), \rvx_q)\bigr] - \KL\bigl(q_2(\rvz_0) \| p(\rvz_0)\bigr) + \mathcal{H}[q_1],
\end{equation}
where $\mathcal{H}[q_1]$ denotes the entropy of $q_1$.

\paragraph{Intuition.} The ELBO in \cref{eq:elbo-joint} balances three objectives: (i) the expected log-likelihood encourages the latent thought and the trace to be mutually consistent (the thought should explain the trace, and the trace should be plausible under the thought); (ii) the KL term anchors the latent to the prior, preventing degenerate solutions; and (iii) the entropy term encourages exploration over traces, avoiding premature commitment to a single solution.

\subsection{Coordinate Ascent Recovers the Rethinking Procedure}

We now show that the two steps of inference-time rethinking (\cref{eq:itr-step1,eq:itr-step2}) arise as coordinate-ascent updates on the ELBO in \cref{eq:elbo-joint}, under delta (point-mass) approximations.

\begin{proposition}[Inference-Time Rethinking as Coordinate Ascent]
\label{prop:rethinking}
Consider the variational objective in \cref{eq:kl-objective} with the mean-field factorization $q(\rvx_r, \rvz_0) = q_1(\rvx_r)\, q_2(\rvz_0)$. The optimal coordinate-ascent updates are:
\begin{enumerate}[label=(\roman*), itemsep=4pt, leftmargin=24pt]
    \item \textbf{Optimal $q_1^*$ (fix $q_2$):} The optimal trace distribution satisfies
    \begin{equation}
    \label{eq:opt-q1}
        \log q_1^*(\rvx_r) = \mathbb{E}_{q_2(\rvz_0)}\bigl[\log p_\beta(\rvx_r \mid U_\alpha(\rvz_0), \rvx_q)\bigr] + \mathrm{const}.
    \end{equation}
    When $q_2 = \delta(\rvz_0 - \hat{\rvz}_0)$ is a point mass, this reduces to the conditional:
    \begin{equation}
    \label{eq:opt-q1-delta}
        q_1^*(\rvx_r) = p_\beta(\rvx_r \mid U_\alpha(\hat{\rvz}_0), \rvx_q),
    \end{equation}
    which is realized by sampling $\rvx_r^{(t)} \sim p_\beta(\rvx_r \mid U_\alpha(\rvz_0^{(t-1)}), \rvx_q)$, i.e., the generation step (\cref{eq:itr-step1}).

    \item \textbf{Optimal $q_2^*$ (fix $q_1$):} The optimal latent distribution satisfies
    \begin{equation}
    \label{eq:opt-q2}
        \log q_2^*(\rvz_0) = \mathbb{E}_{q_1(\rvx_r)}\bigl[\log p_\beta(\rvx_r \mid U_\alpha(\rvz_0), \rvx_q)\bigr] + \log p(\rvz_0) + \mathrm{const}.
    \end{equation}
    When $q_1 = \delta(\rvx_r - \hat{\rvx}_r)$ is a point mass, this yields
    \begin{equation}
    \label{eq:opt-q2-delta}
        q_2^*(\rvz_0) \propto p_\beta(\hat{\rvx}_r \mid U_\alpha(\rvz_0), \rvx_q)\, p(\rvz_0),
    \end{equation}
    which is the true posterior $p_\theta(\rvz_0 \mid \hat{\rvx}_r, \rvx_q)$. Approximating this posterior via ELBO maximization recovers the reflection step (\cref{eq:itr-step2}).
\end{enumerate}
\end{proposition}

\begin{proof}
This follows from the standard mean-field coordinate-ascent derivation~\citep{bishop2006pattern, blei2017variational}. For a factored approximation $q(\rvx_r, \rvz_0) = q_1(\rvx_r)\, q_2(\rvz_0)$, the functional derivative of the ELBO with respect to $q_1$, holding $q_2$ fixed, yields the optimality condition:
\begin{equation}
    \log q_1^*(\rvx_r) = \mathbb{E}_{q_2(\rvz_0)}\bigl[\log p_\theta(\rvx_r, \rvz_0 \mid \rvx_q)\bigr] + \mathrm{const}.
\end{equation}
Expanding the joint using \cref{eq:joint} and absorbing terms independent of $\rvx_r$ into the constant gives \cref{eq:opt-q1}. Under the point-mass assumption $q_2 = \delta(\rvz_0 - \hat{\rvz}_0)$, the expectation collapses, yielding $q_1^*(\rvx_r) \propto p_\beta(\rvx_r \mid U_\alpha(\hat{\rvz}_0), \rvx_q)$, which is already normalized. The derivation for $q_2^*$ is analogous: setting the functional derivative with respect to $q_2$ to zero, absorbing $\rvx_r$-independent terms, and substituting $q_1 = \delta(\rvx_r - \hat{\rvx}_r)$ gives \cref{eq:opt-q2-delta}.
\end{proof}

\paragraph{Intuition.} The two steps have complementary roles. In Step~(i), the model generates the best trace it can given its current understanding of the problem (encoded in $\hat{\rvz}_0$). This is analogous to a student writing out a solution based on their current plan. In Step~(ii), the model examines the trace it just produced and asks: ``what latent thought would best explain this reasoning?'' By optimizing $\rvz_0$ to maximize the likelihood of the generated trace (regularized by the prior), the model revises its internal plan. This is analogous to the student re-reading their work and adjusting their strategy. The alternation allows errors in the trace to inform corrections in the thought, which then guides an improved trace in the next round.

\subsection{Monotonic Convergence}

A key consequence of the coordinate-ascent structure is the following convergence guarantee.

\begin{proposition}[Monotonic ELBO Improvement]
\label{prop:monotone}
Let $\{(q_1^{(t)}, q_2^{(t)})\}_{t=0}^{T_\mathrm{rethink}}$ be the sequence of variational distributions produced by the rethinking procedure. If each coordinate update does not decrease the ELBO, i.e., the $q_1$-update in \cref{eq:opt-q1-delta} and the $q_2$-update in \cref{eq:opt-q2-delta} each (approximately) maximize $\mathcal{L}$ with respect to their factor, then:
\begin{equation}
\label{eq:monotone}
    \mathcal{L}(q_1^{(t+1)}, q_2^{(t+1)}) \geq \mathcal{L}(q_1^{(t)}, q_2^{(t)}) \quad \text{for all } t \geq 0.
\end{equation}
Equivalently, the KL divergence to the joint posterior is monotonically non-increasing:
\begin{equation}
    \KL\bigl(q_1^{(t+1)} q_2^{(t+1)} \| p_\theta(\rvx_r, \rvz_0 \mid \rvx_q)\bigr) \leq \KL\bigl(q_1^{(t)} q_2^{(t)} \| p_\theta(\rvx_r, \rvz_0 \mid \rvx_q)\bigr).
\end{equation}
\end{proposition}

\begin{proof}
Each coordinate update maximizes the ELBO over one factor while holding the other fixed. Since the ELBO is the negative of the KL (up to the constant $\log p_\theta(\rvx_q)$), and each update either increases or preserves the ELBO, the sequence $\{\mathcal{L}^{(t)}\}$ is non-decreasing. Since $\mathcal{L} \leq \log p_\theta(\rvx_q)$ is bounded above, the sequence converges.
\end{proof}

\paragraph{Intuition.} This result provides a principled justification for why more rethinking iterations consistently help (\cref{tab:main}): each round of generate-then-reflect is guaranteed to bring the variational approximation closer to (or at least no farther from) the true joint posterior over thoughts and traces. The model is not randomly searching; it is systematically descending a well-defined objective.

\subsection{Connections and Remarks}
\label{sec:connections}

\paragraph{Connection to Variational EM.} The rethinking procedure can be viewed as a form of variational EM~\citep{neal1998view} applied at test time with fixed model parameters $\theta$. The generation step (updating $q_1$) corresponds to an ``E-step'' that imputes the missing trace, while the reflection step (updating $q_2$) corresponds to a second ``E-step'' that refines the latent variable. Both steps optimize the same ELBO, making this a pure inference procedure rather than a learning algorithm.

\paragraph{Connection to the Wake-Sleep Algorithm.} There is a structural parallel with wake-sleep~\citep{hinton1995wake}: the generation step (sampling from the generative model) resembles the ``sleep'' phase, while the reflection step (inferring latent variables from observations) resembles the ``wake'' phase. However, a key difference is that our procedure optimizes a single consistent objective (the ELBO), whereas the classical sleep phase minimizes a reverse KL that is not jointly consistent with the wake objective.

\paragraph{Approximation Quality.} Two sources of approximation merit discussion. First, the point-mass (delta) approximation for $q_1$ and $q_2$ sacrifices distributional uncertainty for computational efficiency; a richer family (e.g., Gaussian $q_2$, sampling multiple traces for $q_1$) would yield tighter bounds. Second, the reflection step uses finite gradient steps rather than exact optimization, so the per-step ELBO increase is approximate. 

\paragraph{From Likelihood to Reward.} The current formulation uses likelihood as a proxy for correctness, which is effective when training data consists of expert demonstrations (\cref{sec:exp}). The variational framework naturally accommodates extensions: replacing the log-likelihood with an external reward signal $R(\rvx_r)$ (e.g., from a verifier or symbolic checker) transforms the reflection step into test-time policy gradient optimization over the latent space, while the KL regularizer to the prior serves as a trust region~\citep{schulman2017proximal} that keeps the optimization on the learned manifold.

\bibliography{iclr2026_conference}

@book{bishop2006pattern,
  title={Pattern Recognition and Machine Learning},
  author={Bishop, Christopher M.},
  year={2006},
  publisher={Springer}
}

@incollection{neal1998view,
  title={A View of the {EM} Algorithm that Justifies Incremental, Sparse, and Other Variants},
  author={Neal, Radford M. and Hinton, Geoffrey E.},
  booktitle={Learning in Graphical Models},
  editor={Jordan, Michael I.},
  pages={355--368},
  year={1998},
  publisher={Kluwer Academic Publishers}
}

@article{hinton1995wake,
  title={The Wake-Sleep Algorithm for Unsupervised Neural Networks},
  author={Hinton, Geoffrey E. and Dayan, Peter and Frey, Brendan J. and Neal, Radford M.},
  journal={Science},
  volume={268},
  number={5214},
  pages={1158--1161},
  year={1995},
  publisher={American Association for the Advancement of Science}
}

@article{schulman2017proximal,
  title={Proximal Policy Optimization Algorithms},
  author={Schulman, John and Wolski, Filip and Dhariwal, Prafulla and Radford, Alec and Klimov, Oleg},
  journal={arXiv preprint arXiv:1707.06347},
  year={2017}
}

@article{li2025latentseek,
  title={Seek in the Dark: Reasoning via Test-Time Instance-Level Policy Gradient in Latent Space},
  author={Li, Hengli and Li, Chenxi and Wu, Tong and Zhu, Xuekai and Wang, Yuxuan and Yu, Zhaoxin and Jiang, Eric Hanchen and Zhu, Song-Chun and Jia, Zixia and Wu, Ying Nian and Zheng, Zilong},
  journal={arXiv preprint arXiv:2505.13308},
  year={2025}
}

@inproceedings{kong2025latent,
  title     = {Latent Thought Models with Variational Bayes Inference-Time Computation},
  author    = {Deqian Kong and Minglu Zhao and Dehong Xu and Bo Pang and Shu Wang and
               Edouardo Honig and Zhangzhang Si and Chuan Li and Jianwen Xie and
               Sirui Xie and Ying Nian Wu},
  booktitle = {42nd International Conference on Machine Learning (ICML)},
  year      = {2025}
}

@inproceedings{noh2025latent,
  title     = {Latent Adaptive Planner for Dynamic Manipulation},
  author    = {Noh, Donghun and Kong, Deqian and Zhao, Minglu and Lizarraga, Andrew and Xie, Jianwen and Wu, Ying Nian and Hong, Dennis},
  booktitle = {Proceedings of the 9th Conference on Robot Learning (CoRL)},
  series    = {Proceedings of Machine Learning Research},
  volume    = {305},
  pages     = {2430--2448},
  year      = {2025},
}

@article{qin2025generative,
  title={Generative Actor Critic},
  author={Qin, Aoyang and Kong, Deqian and Wang, Wei and Wu, Ying Nian and Zhu, Song-Chun and Xie, Sirui},
  journal={arXiv preprint arXiv:2512.21527},
  year={2025}
}

@article{friederici2011brain,
  title={The brain basis of language processing: from structure to function},
  author={Friederici, Angela D},
  journal={Physiological reviews},
  volume={91},
  number={4},
  pages={1357--1392},
  year={2011},
  publisher={American Physiological Society}
}

@article{mcclelland1995complementary,
  title={Why there are complementary learning systems in the hippocampus and neocortex: insights from the successes and failures of connectionist models of learning and memory},
  author={McClelland, James L and McNaughton, Bruce L and O'Reilly, Randall C},
  journal={Psychological review},
  volume={102},
  number={3},
  pages={419},
  year={1995},
  publisher={American Psychological Association}
}

@article{touvron2023llama,
  title={Llama 2: Open foundation and fine-tuned chat models},
  author={Touvron, Hugo and Martin, Louis and Stone, Kevin and Albert, Peter and Almahairi, Amjad and Babaei, Yasmine and Bashlykov, Nikolay and Batra, Soumya and Bhargava, Prajjwal and Bhosale, Shruti and others},
  journal={arXiv preprint arXiv:2307.09288},
  year={2023}
}

@article{patel2021nlp,
  title={Are NLP models really able to solve simple math word problems?},
  author={Patel, Arkil and Bhattamishra, Satwik and Goyal, Navin},
  journal={arXiv preprint arXiv:2103.07191},
  year={2021}
}

@inproceedings{koncel2016mawps,
  title={MAWPS: A math word problem repository},
  author={Koncel-Kedziorski, Rik and Roy, Subhro and Amini, Aida and Kushman, Nate and Hajishirzi, Hannaneh},
  booktitle={Proceedings of the 2016 conference of the north american chapter of the association for computational linguistics: human language technologies},
  pages={1152--1157},
  year={2016}
}

@inproceedings{roy2015solving,
  title={Solving general arithmetic word problems},
  author={Roy, Subhro and Roth, Dan},
  booktitle={Proceedings of the 2015 conference on empirical methods in natural language processing},
  pages={1743--1752},
  year={2015}
}

@article{liu2025marcos,
  title={Marcos: Deep thinking by markov chain of continuous thoughts},
  author={Liu, Jiayu and Huang, Zhenya and Sims, Anya and Chen, Enhong and Teh, Yee Whye and Miao, Ning},
  journal={arXiv preprint arXiv:2509.25020},
  year={2025}
}

@article{shen2025codi,
  title={Codi: Compressing chain-of-thought into continuous space via self-distillation},
  author={Shen, Zhenyi and Yan, Hanqi and Zhang, Linhai and Hu, Zhanghao and Du, Yali and He, Yulan},
  journal={arXiv preprint arXiv:2502.21074},
  year={2025}
}

@article{deng2024explicit,
  title={From explicit cot to implicit cot: Learning to internalize cot step by step},
  author={Deng, Yuntian and Choi, Yejin and Shieber, Stuart},
  journal={arXiv preprint arXiv:2405.14838},
  year={2024}
}

@article{team2024qwen2,
  title={Qwen2 technical report},
  author={Team, Qwen and others},
  journal={arXiv preprint arXiv:2407.10671},
  volume={2},
  number={3},
  year={2024}
}

@article{tan2025think,
  title={Think Silently, Think Fast: Dynamic Latent Compression of LLM Reasoning Chains},
  author={Tan, Wenhui and Li, Jiaze and Ju, Jianzhong and Luo, Zhenbo and Luan, Jian and Song, Ruihua},
  journal={arXiv preprint arXiv:2505.16552},
  year={2025}
}

@article{deng2023implicit,
  title={Implicit chain of thought reasoning via knowledge distillation},
  author={Deng, Yuntian and Prasad, Kiran and Fernandez, Roland and Smolensky, Paul and Chaudhary, Vishrav and Shieber, Stuart},
  journal={arXiv preprint arXiv:2311.01460},
  year={2023}
}

@inproceedings{vaswani2017attention,
  title={Attention is all you need},
  author={Vaswani, Ashish and Shazeer, Noam and Parmar, Niki and Uszkoreit, Jakob and Jones, Llion and Gomez, Aidan N and Kaiser, {\L}ukasz and Polosukhin, Illia},
  booktitle={Advances in Neural Information Processing Systems},
  volume={30},
  pages={5998--6008},
  year={2017}
}

@article{kingma2014adam,
  title={Adam: A method for stochastic optimization},
  author={Kingma, Diederik P and Ba, Jimmy},
  journal={arXiv preprint arXiv:1412.6980},
  year={2014}
}

@article{jordan1999introduction,
  title={An introduction to variational methods for graphical models},
  author={Jordan, Michael I and Ghahramani, Zoubin and Jaakkola, Tommi S and Saul, Lawrence K},
  journal={Machine learning},
  volume={37},
  number={2},
  pages={183--233},
  year={1999},
  publisher={Springer}
}

@article{ullman2004contributions,
  title={Contributions of memory circuits to language: The declarative/procedural model},
  author={Ullman, Michael T},
  journal={Cognition},
  volume={92},
  number={1-2},
  pages={231--270},
  year={2004},
  publisher={Elsevier}
}

@article{blei2017variational,
  title={Variational inference: A review for statisticians},
  author={Blei, David M and Kucukelbir, Alp and McAuliffe, Jon D},
  journal={Journal of the American Statistical Association},
  volume={112},
  number={518},
  pages={859--877},
  year={2017}
}

@inproceedings{bowman2016generating,
  title={Generating sentences from a continuous space},
  author={Bowman, Samuel R and Vilnis, Luke and Vinyals, Oriol and Dai, Andrew M and Jozefowicz, Rafal and Bengio, Samy},
  booktitle={Proceedings of the 20th SIGNLL Conference on Computational Natural Language Learning},
  pages={10--21},
  year={2016}
}

@article{kingma2013auto,
  title={Auto-encoding variational bayes},
  author={Kingma, Diederik P and Welling, Max},
  journal={arXiv preprint arXiv:1312.6114},
  year={2013}
}

@article{wei2022chain,
  title={Chain-of-thought prompting elicits reasoning in large language models},
  author={Wei, Jason and Wang, Xuezhi and Schuurmans, Dale and Bosma, Maarten and Xia, Fei and Chi, Ed and Le, Quoc V and Zhou, Denny and others},
  journal={Advances in neural information processing systems},
  volume={35},
  pages={24824--24837},
  year={2022}
}

@article{cobbe2021training,
  title={Training verifiers to solve math word problems},
  author={Cobbe, Karl and Kosaraju, Vineet and Bavarian, Mohammad and Chen, Mark and Jun, Heewoo and Kaiser, Lukasz and Plappert, Matthias and Tworek, Jerry and Hilton, Jacob and Nakano, Reiichiro and others},
  journal={arXiv preprint arXiv:2110.14168},
  year={2021}
}

@article{hao2024training,
  title={Training large language models to reason in a continuous latent space},
  author={Hao, Shibo and Sukhbaatar, Sainbayar and Su, DiJia and Li, Xian and Hu, Zhiting and Weston, Jason and Tian, Yuandong},
  journal={arXiv preprint arXiv:2412.06769},
  year={2024}
}

@article{kong2024latent,
  title={Latent plan transformer for trajectory abstraction: Planning as latent space inference},
  author={Kong, Deqian and Xu, Dehong and Zhao, Minglu and Pang, Bo and Xie, Jianwen and Lizarraga, Andrew and Huang, Yuhao and Xie, Sirui and Wu, Ying Nian},
  journal={Advances in Neural Information Processing Systems},
  volume={37},
  pages={123379--123401},
  year={2024}
}

@article{kong2024molecule,
  title={Molecule design by latent prompt transformer},
  author={Kong, Deqian and Huang, Yuhao and Xie, Jianwen and Honig, Edouardo and Xu, Ming and Xue, Shuanghong and Lin, Pei and Zhou, Sanping and Zhong, Sheng and Zheng, Nanning and others},
  journal = {Advances in Neural Information Processing Systems},
 pages = {89069--89097},
 volume = {37},
 year = {2024}

}
\bibliographystyle{iclr2026_conference}


\end{document}